\documentclass[preprint,authoryear,review]{elsarticle}

\usepackage{lineno}
\usepackage[hidelinks,breaklinks=true]{hyperref}

\usepackage{algorithmic}
\usepackage{algorithm}
\usepackage{color}
\usepackage{subcaption}
\usepackage{array}
\usepackage{comment}

\usepackage{multirow}
\usepackage{amsmath}
\usepackage{algorithmic,amsthm}
\usepackage{graphicx}
\usepackage{caption}
\usepackage{subcaption}
\usepackage{placeins}

\usepackage{algorithmic}
\usepackage{algorithm}
\usepackage{color}
\usepackage{subcaption}
\usepackage{array}
\usepackage{bm}
\usepackage{multirow}

\usepackage{amssymb}

\newtheorem{theorem}{Theorem}

\def\Zset{{\mathcal Z}}

\def\Iset{{\mathcal I}}
\def\Mset{{\mathcal M}}
\def\Gset{{\mathcal G}}
\def\Pset{{\mathcal P}}
\newcommand\Pai{\textit{PA}}

\newcommand\bs[1]{\boldsymbol{#1}}

\DeclareMathOperator*{\argmax}{argmax}

\modulolinenumbers[5]

\journal{NeuroImage}

\makeatletter
\def\ps@pprintTitle{%
 \let\@oddhead\@empty
 \let\@evenhead\@empty
 \def\@oddfoot{}%
 \let\@evenfoot\@oddfoot}
\makeatother

\begin{document}

\begin{frontmatter}

\title{Learning Bayesian Networks with \\Incomplete Data by Augmentation}

\author[mainaddress01]{Tameem Adel\corref{mycorrespondingauthor}}
\cortext[mycorrespondingauthor]{Corresponding author}
\ead{tameem.hesham@gmail.com}

\author[mainaddress03]{Cassio P.\ de Campos}

\address[mainaddress01]{Machine Learning Lab, University of Amsterdam}
\address[mainaddress03]{EEECS, Queen’'s University Belfast}

\begin{abstract}
We present new algorithms for learning Bayesian networks from data with missing values using a data augmentation approach. An exact Bayesian network learning algorithm is obtained by recasting the problem into a standard Bayesian network learning problem without missing data. To the best of our knowledge, this is the first exact algorithm for this problem. As expected, the exact algorithm does not scale to large domains. We build on the exact method to create an approximate algorithm using a hill-climbing technique. This algorithm scales to large domains so long as a suitable standard structure learning method for complete data is available. We perform a wide range of experiments to demonstrate the benefits of learning Bayesian networks with such new approach.
\end{abstract}

\end{frontmatter}


\section{Introduction} \label{sec:intro}
Missing entries in real-world data exist due to various reasons. For instance, it can be due to damage of the device used to record feature values; a metal detector might fail to produce a signal denoting the existence of a metal due to a certain malfunction. Results can be incomplete in an industrial experiment due to mechanical breakdowns not necessarily related to the performed experiment~\citep{Little2014}. Recommendation data can have missing values since participants in the recommendation system did not rate all the available songs, films, books, etc. While data missingness in the above examples can mostly be assumed to be generated by a random process which depends only on the observed data, usually referred to as \emph{missing at random (MAR)}~\citep{Little2014,Rancoita2016}, this assumption might fail in other examples. People seeking for health insurance might refuse to give an answer to certain questions in order to reduce the costs, e.g. `do you smoke?', and in many cases this can be seen as an indication of one specific answer. In such cases we say that data are {\it missing not at random}, or MNAR (see for instance~\citep{broeck14}).

Given a dataset with categorical random variables, the Bayesian network structure learning problem refers to finding the best network structure (a directed acyclic graph, or DAG) according to a score function based on the data~\citep{Heckerman1995}. As well known, learning a Bayesian network from complete data is NP-complete~\citep{Chickering1996}, and the task becomes even harder with incomplete data. In spite of that, the problem of learning a Bayesian network from incomplete data by (an optimistic) augmentation belongs to the same complexity class, as we will show later on. Because of such result, we investigate and obtain a new exact algorithm for the problem, based on reformulating it into a standard structure learning without missing data. This is the first exact algorithm for the problem, to the best of our knowledge. In contrast to previous work, our algorithm performs both tasks, namely structure learning and data imputation, in a single shot rather than learning the Bayesian network and then dealing with the missing data, possibly in an iterative manner~\citep{Friedman1998,Rancoita2016}. 
Based on the optimization that is required to solve the problem and on the exact algorithm, we devise a hill-climbing approximate algorithm. The hill-climbing regards the completions of the missing values only, while the structure optimization is performed by any off-the-shelf algorithm for structure learning under complete data.

Most previous work to learn the structure of Bayesian networks from incomplete data has focused on MAR. The seminal algorithm in \cite{Friedman1998} introduced an iterative method based on the Expectation-Maximization (EM) technique, referred to as structural EM. Implementation of structural EM begins with an initial graph structure, followed by steps where the probability distribution of variables with missing values is estimated by EM, alternated with steps in which the expectation of the score of each neighbouring graph is computed. After convergence, the graph maximizing the score is chosen. Many other algorithms have used ideas from structural EM and deal separately with the missing values and the structure optimization using complete data~\citep{Borchani2006,Leray2005,Meila1998,Ramoni1997,Riggelsen2006,Riggelsen2005}. In \cite{Rancoita2016}, structures are learned from incomplete data using a structural EM whose maximization step is performed by an anytime method, and the `expectation' step imputes the missing values using expected means, or modes, of the current estimated joint distribution. By using modes in each iteration~\citep{Ramoni1997}, the EM method is sometimes called {\it hard} EM, and is close to our work. In some sense, we work with a global optimization version of hard EM. While this is not exactly considering data to be MNAR, such approach fits less the observed data and performs well for MNAR missing data when compared to structural EM, as we will empirically show. We emphasize that the actual missingness process is not disclosed to the methods and is not assumed to be somehow known, and that we are mainly interested in structure learning. Given the difficulties of structure learning itself, we assume that the underlying distribution is identifiable (in short terms, provided enough data are available, one could reconstruct such distribution, see for instance~\citep{Mohan2013}).

We perform experiments on a set of heterogeneous datasets. We base the evaluation on imputation accuracy in its pure form, as well as in the forms of classification accuracy and semi-supervised learning accuracy. Experiments show the improvements achieved by the proposed algorithms in all scenarios. Regarding the comparison between our exact and approximate methods, experiments suggest that accuracy levels achieved by the approximate algorithm are close to those achieved by the optimal learning algorithm, with the former being much faster and scalable.

\section{Bayesian Network Structure Learning}

Let ${\bf X}=(X_1,\ldots,X_m)$ refer to a vector of categorical random variables, taking values in $\textrm O_{{\bf X}}=\times_i \textrm O_{X_i}$, where $\textrm O_{{\bf X}}$ represents the Cartesian product of the state space, $\textrm O_{X_i}$, of each $X_i$. 
Denote by $\mathcal{D}$ an $n$-instance dataset where each instance
${\bf D}_u = (d_{u,1}, d_{u,2},\dots, d_{u,m})$ is such that $d_{u,i}$ is either an observed value $o_{u,i}\in O_{X_i}$ or a special symbol denoting the entry is missing.
Let ${\bf Z}_u$ denote a completion for variables with missing values in instance $u$ and $z_{u,i}$ for the missing value of $X_i$.

A Bayesian network, $\Mset$, is a probabilistic graphical model based on a structured dependency among random variables to represent a joint probability distribution in a compact and tractable manner. Here, it represents a joint probability distribution $\text{Pr}_\Mset$ over a collection of categorical random variables, ${\bf X}$. We define a Bayesian network as a triple $\Mset=(\Gset,{\bf X},\Pset)$, where $\Gset=(V_\Gset,E_\Gset)$ is a directed acyclic graph (DAG) with $V_\Gset$ a collection of $m$ nodes associated to the random variables ${\bf X}$ (a node per variable), and $E_\Gset$ a collection of arcs; $\Pset$ is a collection of conditional probabilities $\text{Pr}_\Mset(X_i|\textit{PA}_i)$ where $\textit{PA}_i$ denotes the parents of $X_i$ in the graph ($\textit{PA}_i$ may be empty), corresponding to the relations of $E_\Gset$. In a Bayesian network, the Markov condition states that every variable is conditionally independent of its non-descendants given its parents. This structure induces a joint probability distribution by the expression $\text{Pr}_\Mset(X_1,\dots,X_m) = \prod_i \text{Pr}_\Mset(X_i|\textit{PA}_i)$. We define $r_i\geq 2$ as the number of values in $\textrm O_{X_i}$,  i.e.\ $r_i = |O_{X_i}|$, and $r_{\Pai_i}$ as the number of possible realizations of the parent set, that is, $r_{\Pai_i}=\prod_{X_l\in \Pai_i} r_l$. Let $R=\max_i r_i$.

Given a complete dataset $\mathcal{D}$ with $n$ instances, the structure learning problem in Bayesian networks is to find a DAG $\Gset$ that maximizes a given score function, that is, we look for $\Gset^*=\argmax_{\Gset\in {\bs\Gset}} s_{\mathcal{D}}(\Gset)$, with ${\bs\Gset}$ the set of all DAGs over node set ${\bf X}$.
We consider here the score function $s_{\mathcal{D}}$ to be the Bayesian Dirichlet Equivalent Uniform (BDeu) criterion \citep{Buntine1991,Cooper1992} (other decomposable scores could be used too), so we have $s_{\mathcal{D}}(\Gset)=\sum_i s_{\mathcal{D}}(X_i,\textit{PA}_i)$.
We however have to deal with the missing part of the data, which we treat by completing the missing values in the best possible way (an optimistic completion):
\begin{align}
& (\Gset^*,\Zset^*)=\argmax_{\Gset\in {\bs\Gset},~\Zset\in {\bs\Zset}} s_{\mathcal{D}}(\Gset,\Zset) = \notag \\
& \argmax_{\Gset\in {\bs\Gset},~\Zset\in {\bs\Zset}} \sum_i s_{\mathcal{D}}(X_i,\textit{PA}_i;\textbf{Z}_{\{X_i\}\cup\textit{PA}_i})
\label{eq1}
\end{align}
\noindent where ${\bs\Zset}=\times_u O_{{\bf Z}_u}$ and $s_{\mathcal{D}}(\Gset,\Zset)$ is the score $s_{\mathcal{D}}(\Gset)$ evaluated for the complete data when its missing values are replaced by $\Zset$, while $s_{\mathcal{D}}(X_i,\textit{PA}_i;\textbf{Z}_{\{X_i\}\cup\textit{PA}_i})$ is the local score for a node $X_i$ with parent set $\textit{PA}_i$ (note that such computation only depends on the completion $\textbf{Z}_{\{X_i\}\cup\textit{PA}_i}$ of the involved variables). We refer to this optimization task as the structure learning problem by optimistic augmentation. It can be applied to MAR data, but we argue that it is particularly suitable to MNAR when compared to the standard techniques such as structural EM. From the optimization viewpoint, this can be seen as a global optimization approach to {\it hard} EM, since we complete the data with their mode, but we do it globally instead of in an iterative process such as EM. As well known, {\it hard} EM can be seen as a subcase of EM, since it is equivalent to allowing EM to use only {\it degenerate} mass functions in its expectation step. 

\begin{theorem}
The decision version associated to the structure learning problem by optimistic augmentation is NP-complete.
\end{theorem}
\begin{proof}
Hardness is obtained by realizing that this problem generalizes the structure learning problem without missing data, which is NP-hard~\citep{Chickering1996}. Pertinence in NP holds since given $\Gset$ and $\Zset$, the score function $s_{\mathcal{D}}$ can be computed in polynomial time.
\end{proof}

Since the problem is a combinatorial optimization over a discrete domain (both DAGs and completions of data are discrete entities), we could resort to enumerating all possible solutions. This is obviously infeasible for both: the number of DAGs grows super-exponentially in the number of variables and the number of completions grows exponentially in the number of missing values. 
We will now present an exact algorithm for the problem which transforms it into a standard structure learning problem, and later we modify the approach to perform approximate learning. In this respect, we define as a $t$-local optimal solution for Equation~\eqref{eq1}  a pair $(\Gset,\Zset)$ such that $s_D(\Gset,\Zset)\geq s_D(\Gset',\Zset')$ for all $\Gset'$ and all $\Zset'$ with $\textrm{HD}(\Zset,\Zset')\leq t$, where $\textrm{HD}$ is the Hamming distance, that is, $(\Gset,\Zset)$ is optimal with respect to any other pair whose completion of the data has at most $t$ elements different from $\Zset$. A global optimal solution is a $\infty$-local optimal solution.


\subsection{Optimal (Exact) Learning Algorithm}

We assume that a standard structure learning algorithm for complete data is available to us, which is based on the framework of two main optimizations: (i) parent set identification and (ii) structure optimization. Step (i) concerns building a list of candidate parent sets for each variable, while Step (ii) optimizes the selection of a parent set for each variable in a way to maximize the total score while ensuring that the graph is a DAG. This latter step can be tackled by exact or approximate methods~\citep{Bartlett2013,Scanagatta2015} (in our experiments we will employ an exact method such that we are sure that the quality of results is only affected/related to the proper treatment of the missing data, but for very large domains any approximate method could be used too).

The exact algorithm for solving Equation~\eqref{eq1} is based on modifying the parent set identification step. This step has no known polynomial-time solution if we do not impose a maximum number of parents~\citep{Koivisto2006b}, so we will assume that such a bound $k$ is given. We compute the candidate list by using one of the available approaches~\citep{Campos2011,Scanagatta2015} to guide the search, but for each candidate to be evaluated, the corresponding variables in the dataset might contain missing values.
The first part of the transformation is to create gadgets composed of some new artificial variables which will be related to the missing values and will enable the inclusion of all possible replacements of missing values by augmenting the original domain.

Over all the dataset, for each and every missing value, let us denote it by $(u,i)$ for sample $u$ and variable $X_i$, we include artificial variables $X_{(u,i),1},\ldots,X_{(u,i),r_i}$. Each $X_{(u,i),j}$ has two parent set candidates: 
(i) ${\bf X}\cup\{X_{(u,i),1+(j\textrm{ mod } r_u)}\}$ with score zero (assuming all other score values are negative, without loss of generality) and (ii) $\emptyset$ with score $-\lambda$, with $\lambda$ a large enough value (e.g. greater than the sum of all other absolute scores).
We further illustrate the idea via an example for variable $X_1$ with $r_1 = 3$: Assume $m=3$, $r_1=3$ and there is one missing value at $(u,1)$. An artificial variable is included for each possible completion $z_{u,1}$, resulting in a total of three new variables, $X_{(u,1),1}, X_{(u,1),2}, X_{(u,1),3}$. The following gadget, consisting of two parent set candidates per artificial variable, is added to the list of parent set scores (we know that only one parent set per variable will be chosen during the optimization phase later on):\\
\begin{align}
 &s(X_{({u,1}),1}, \{X_{({u,1}),2}, X_1, X_2, X_3 \}) = 0, \notag\\ 
 &s(X_{({u,1}),1}, \emptyset) = -\lambda, \notag \\
 &s(X_{({u,1}),2}, \{X_{({u,1}),3}, X_1, X_2, X_3 \}) = 0, \notag\\ 
 &s(X_{({u,1}),2}, \emptyset) = -\lambda, \notag \\
 &s(X_{({u,1}),3}, \{X_{({u,1}),1}, X_1, X_2, X_3 \}) = 0, \notag\\ 
&s(X_{({u,1}),3}, \emptyset) = -\lambda. \notag
\end{align}

According to this gadget, each artificial variable will either have no parent variables or all other original variables as well as one other artificial variable as its set of parents. The case with no parents leaves open the opportunity to choose the variable representing such completion as a potential parent for all original variables. In contrast, the cases with all variables as parents disables such completion from being chosen as a parent by the original variables, otherwise it would create a cycle. 
Due to including one artificial variable as a parent of the next artificial variable, at least one parent set among those with score zero cannot be chosen (otherwise a cycle is formed), and because they are all very good scores when compared to $-\lambda$, all but one will certainly be chosen. There is one such gadget per missing value in the original dataset, so we spend time $O(R\cdot m\cdot C)$, where $C$ is the number of missing values.

Finally, we return to the computation of the score for a given variable and parent set. Let $X_i$ be the variable of interest and $\textit{PA}_i=\{X_{i_1},\ldots,X_{i_q}\}$ for which the score must be evaluated. At this moment, we consider all possible completions $\textbf{Z}_{\{X_i\}\cup\textit{PA}_i}$ and compute the scores $s_{\mathcal{D}}(X_i,\textit{PA}_i;\textbf{Z}_{\{X_i\}\cup\textit{PA}_i})$ for each one of them. In order to reduce the problem to a standard structure learning without missing data, we must index these scores somehow. This is made possible via the new artificial variables:
\begin{align}
s_{\mathcal{D}}&(X_i,\textit{PA}_i;\textbf{Z}_{\{X_i\}\cup\textit{PA}_i}) = \notag \\ 
&s_{\mathcal{D}}(X_i,\textit{PA}_i\cup\{X_{(u,i),z_{u,j}}:~ z_{u,j}\in \textbf{Z}_{\{X_i\}\cup\textit{PA}_i}\}) \notag
\end{align}
\noindent that is, for each imputed missing value $z_{u,j}$ appearing for variable $X_i$ or $\textit{PA}_i$ we will have an extra parent within the parent set that tells which completion was used for that missing value, according to the completion $\textbf{Z}_{\{X_i\}\cup\textit{PA}_i}$. This idea is applied to every evaluation of the score of a parent set, for every possible completion $\textbf{Z}_{\{X_i\}\cup\textit{PA}_i}$, so the final list of candidates will include only parent sets for which the completion of the data is `known' at the time that the score is computed. In order to ensure that the completions are compatible among different local score computations, the gadgets explained before are enough, since they force that a certain completion be chosen for each missing value.

\begin{theorem}
The exact algorithm transforms the structure learning problem by augmentation into a standard structure learning without missing data in time $O(R\cdot m\cdot C)$, plus time $O(n\cdot k\cdot R^c)$ per parent set evaluation, where $C$ is the total number of missing values and $c$ is the maximum number of missing values appearing in the variable of interest or in variables in the parent set being evaluated (hence polynomial in all parameters but $c$).
\end{theorem}

There will be many score computations and entries in the list, exponential in the number of missing values involved. So the benefit of this approach is that usually only a few variables are involved in the score computation at the same time. The drawback is that it cannot handle datasets with many missing values for the same variable, since it is $R^c$ times slower than the corresponding parent set evaluation without missing data. Next we address this issue by proposing an approximate method (the exact method is nevertheless useful in small domains and also  important to check whether the approximate version achieves reasonable results).

\subsection{Approximate Algorithm}

Albeit locally to the variables involved in the evaluation of a parent set, the exact method considers all possible completions of the data. This is fine with a few missing values per variable, but if there are many missing values, in particular within the same variable, the exact method becomes computationally infeasible. We propose an approximate algorithm based on a hill-climbing idea. We start with an initial guess $\Zset_0$ (or several different random guesses) for the completion of all missing values in the dataset. Then we execute the very same steps of the exact algorithm, but we restrict the completions only to those which are at most $t$ elements different from the current guess $\Zset_h$. There are at most $(R\cdot m)^t$ completions $\Zset'_h$ such that $\textrm{HD}(\Zset_h,\Zset'_h)\leq t$. We proceed as with the exact method, but applying such constraint during the transformation that was explained in the previous section. After the transformation is done, the structure optimization is run and a new structure and new data completion $\Zset_{h+1}$ is obtained. We repeat the process until convergence, that is, until $\Zset_{h+1}=\Zset_h$.

\begin{theorem}
The approximate algorithm transforms the structure learning problem by augmentation into a standard structure learning without missing data in time $O(R\cdot m\cdot C)$, plus time $O(n\cdot k\cdot (R\cdot m)^t)$ per parent set evaluation ($C$ is the total number of missing values and $t$ is the amount of locality of the approximation, as previously defined), that is, polynomial in all parameters but $t$.
\end{theorem}

The outcome of the approximate learning algorithm is the network structure as well as the completion of all the missing data values. The approximate algorithm might lead to a locally optimal solution, but on the other hand it is much more scalable than the exact algorithm. 

\begin{theorem}
Provided that an optimal structure learning optimization algorithm is available, the approximate algorithm always converges to a $t$-local optimal solution.
\end{theorem}

If we want to scale to very large domains, we could also resort to an approximate structure learning optimization algorithm (e.g.~\citep{Scanagatta2015}). In this case, our approximate algorithm could be used in domains with hundreds or even thousands of variables (using very small $t$), but we would lose the guarantee to converge to a $t$-local optimal solution (it would still be a local optimum, but we would have to define it locally also in terms of the graph structures).

\section{Experiments}
We perform experiments on simulated as well as real-world data. The main evaluation metric used is accuracy of the imputation of missing data values, either in the form of missing values spread throughout the data, or in the form of a binary classification problem where only the class variable can contain missing values. Most of our experiments are with binary data for the sake of exposition, even though the algorithms are general and can be used with any categorical data (as shown in the last experimental setting). To test significance, we perform a paired t-test with significance level at $5\%$. Throughout all tables of results, a result in \textbf{bold} refers to an accuracy value that is significantly better than its competitors, whereas showing two results belonging to the same experiment in \textbf{bold} means that each of them being significantly better than the rest of the competitors.
For structure optimization, we use the exact solver referred to as Gobnilp \citep{Bartlett2013} with the code available from \url{https://www.cs.york.ac.uk/aig/sw/gobnilp/}.
We perform comparisons among the two proposed algorithms (exact and approximate) and the structural Expectation-Maximization (EM) algorithm \citep{Friedman1998}. 
We compare accuracy of the three algorithms based on the percentage of correct imputations over all missing values. As for the structural EM, we have used the implementation available at https://github.com/cassiopc/csda-dataimputation~\citep{Rancoita2016}. After convergence, we run the prediction of missing values using a {\it most probable explanation} query. We must emphasize that the task of Bayesian network structure learning with missing values is very challenging, since it is already challenging without missing values. Therefore, we have focused on real but controlled experiments where we can effectively run the algorithms and assert their quality. We use maximum number of parents, $k = 3$, and use $t = 1$.

\subsection{Well-known Bayesian Networks}
We perform experiments using real but small data sets in order to compare both exact and approximate algorithms. First,
we employ the original Bayesian network model for Breast Cancer~\citep{Almeida2014}, which contains $8$ binary variables, we simulate $100$ data instances. That model has been learned from cancer patients of the University of Wisconsin Medical Hospital. Features (Bayesian network nodes) include breast density, mass density, architectural distortion and others, in addition to the diagnosis variable whose binary value refers to benign or malignant~\citep{Dorsi2003}. We include two missing values per variable, resulting in a total of $16$ missing values.
These missing values are generated in a MNAR manner by randomly removing values that are equal to each other, that is, during the generation we enforce that all missing values are zero, or that all missing values are one.
 Imputation results of the proposed exact learning algorithm, approximate algorithm and structural EM are displayed in the first row of Table~\ref{bn-res01} over 100 repetitions of the experiment. 

Second, we use the Bayesian network that has been learned from the Prostate Cancer data by the Tree Augmented Naive Bayes (TAN)~\citep{Friedman1997}, implemented by WEKA \citep{Hall2009}. The Prostate Cancer data were acquired during three different moments in time~\citep{Sarabando20111,Almeida2014}, i.e.\ during a medical appointment, after performing auxiliary exams, and five years after a radical prostatectomy. It contains $11$ binary variables, and $100$ instances are generated. We randomly produce two MNAR missing values per variable, resulting in a total of $22$ missing values. Results are shown in the second row of Table~\ref{bn-res01}. 

Third, the well-known ASIA network is used~\citep{Lauritzen1988}. We generate $100$ instances according to this model, which contains $8$ binary variables. Two missing values are randomly generated according to MNAR. Imputation results are displayed in the third row of Table~\ref{bn-res01}. 
Results indicate that the algorithms proposed here are significantly better than structural EM, which is expected since in this experiment data are not MAR. More interestingly, results of the proposed exact and approximate BN learning algorithms are not significantly different, which supports the use of the (more efficient) approximate method for larger domains. 

\begin{table}[H]
  \caption{Accuracy of imputation for data simulated from different Bayesian networks with two MNAR missing values per variable.}
  \label{bn-res01}
  \centering
  \begin{tabular}{|c|c|c|}
\hline
  Bayesian net & Algorithm  & Average imputation accuracy \\
\hline
   \multirow{ 3}{*} {Breast Cancer} & Exact learning  & \textbf{84.38\%}     \\
   & Approx. learning & \textbf{80\%}     \\
   & Structural EM  & 50\%   \\
\hline
   \multirow{ 3}{*} {Prostate Cancer} & Exact learning  & \textbf{91\%}     \\
   & Approx. learning  & \textbf{86.36\%}     \\
   & Structural EM  & 50\%   \\
\hline
   \multirow{ 3}{*} {ASIA} & Exact learning  & \textbf{84.38\%}     \\
   & Approx. learning  & \textbf{79\%}     \\
   & Structural EM  & 43.75\%   \\
\hline
  \end{tabular}
\end{table}

\subsection{(LUng CAncer Simple set) LUCAS Dataset}\label{sec:lucas}
The LUCAS dataset contains data of the LUCAS causal Bayesian network~\citep{Fogelman2008} with $11$ binary variables, as well as the binary class variable, and
contains $2000$ instances. In this experiment we conduct an analysis of both MAR and MNAR missing data, in order to understand whether the benefits that we
have seen before are only significant in the MNAR case.
Thus, we carry out two experiments: (i) MNAR setting by randomly generating missing values all having the same data value (we repeat that to both zero and one values, one at a time); (ii) \emph{MAR} setting by randomly generating missing values regardless of their respective original values. These simulations are repeated 100 times. 

First, we generate two missing values per variable ($24$ missing values). A comparison between the imputation accuracy values of the approximate algorithm and structural EM is displayed in the first two rows of Table~\ref{bn-res02} (named `Spread All Over'). Surprisingly, our new algorithm is significantly better than structural EM even when missing data are MAR. 

Second, we generate $20$ missing class values and repeat the experiment to span all instances such that each run involves missing values belonging to different instances (without replacement). For the MNAR experiment, each run consists of $20$ identical missing class values (that is, we only make missing values of the same class, and we repeat that for both classes). For the MAR case, there is no such restriction and missing class values are randomly generated. Hence, there are $100$ runs in order to cover all $2000$ instances. Results of the approximate algorithm, structural EM and SVM using different kernels (for the sake of comparison with a state-of-the-art classifier) are displayed in the bottom rows of Table~\ref{bn-res02}. Results of the proposed algorithm are significantly better when MNAR data are used, while the same cannot be stated for the MAR case (accuracy of the proposed algorithm is nevertheless superior to the others in the MAR case).

\begin{table}[H]
  \caption{Accuracy of imputation for experiments performed on the Lung Cancer dataset (LUCAS). \emph{Spread All Over} refers to an imputation of $2$ missing values per variable out of the $12$ LUCAS variables. \emph{Classification} refers to a classification problem performed as a cross-validation (100-fold cross-validation in the MNAR setting case) on LUCAS, using SVM, vs.\ an imputation task on the $20$ missing class variables of the same folds, by both the proposed approximate learning algorithm and Structural\ EM. SVM kernels displayed are those that achieved the highest accuracy in each experiment. MP stands for missingness process, and rbf for radial basis function.}
  \label{bn-res02}
  \centering
  \begin{tabular}{|c|c|c|}
\hline
  MP & Algorithm  & Average imputation accuracy \\
\hline
\multicolumn{3}{|c|}{Exp.: Spread All Over}  \\
\hline
\multirow{2}{*} {MNAR} & Approx. learning & \textbf{70.83\%} \\ \cline{2-3}
& Structural\ EM  & 45\% \\
\cline{1-3}
\multirow{ 2}{*} {MAR} & Approx. learning & 70\% \\ \cline{2-3}
& Structural\ EM & 50\% \\
\hline
\cline{1-3}
\multicolumn{3}{|c|}{Exp.: Classification}  \\
\hline
\multirow{3}{*}{MNAR} &  Approx. learning & \textbf{97.5\%} \\ \cline{2-3}
& Structural\ EM & 42.5\% \\ \cline{2-3}
& SVM (rbf) & 45\% \\
\cline{1-3}
 \multirow{ 3}{*} {MAR} & Approx. learning & $69\%$ \\ \cline{2-3}
& Structural\ EM & 70\% \\ \cline{2-3}
& SVM (rbf) & 55\% \\
\hline
  \end{tabular}
\end{table}

\subsection{SPECT Dataset}
The Single Proton Emission Computed Tomography (SPECT) dataset consists of binary data denoting partial diagnosis from SPECT images \citep{Lichman2013}. Each patient is classified into one of two categories, normal and abnormal. The SPECT data consists of $267$ instances 
and $23$ variables in total ($22$ binary variables and a binary class variable). We generate MNAR missing data with different proportions, always using only one specific value (missing data proportions over all the data are $3\%$, $5\%$ and $10\%$). These randomly generated datasets are given as input to the approximate algorithm as well as to structural EM. We note that there is a large discrepancy in the number of data values holding each of the two binary values: About $67\%$ of the SPECT data has a value $0$, whereas merely $33\%$ of the data has a value $1$. Due to that, we investigate the average MNAR imputation accuracy within each data value separately, and note as well that there is some discrepancy in such accuracy values. Imputation accuracy of the approximate learning algorithm and structural EM are displayed in Table~\ref{bn-res03}. The new algorithm is significantly better.

\begin{table}[ht]
  \caption{MNAR imputation accuracy for the BN Approximate Learning algorithm and Structural EM on the SPECT dataset with various proportions of missing values, and for both data values.}
  \label{bn-res03}
  \centering
  \begin{tabular}{@{}|c|c|c|c|}
\hline
  Missing values \ & Algorithm  & Average imputation accuracy \\
\hline
\multirow{ 2}{*} {$3\%$ (overall)} & New approx. & \textbf{81.75\%} \\ \cline{2-3}
 & Structural\ EM  & 60\% \\
\hline
\multirow{ 2}{*} {$5\%$ (overall)} &New approx. & \textbf{75.22\%} \\ \cline{2-3}
& Structural\ EM  & 49.27\% \\
\hline
\multirow{ 2}{*} {$10\%$ (overall)} & New approx. & \textbf{81.94\%} \\ \cline{2-3}
& Structural\ EM  & 62.04\% \\
\hline
\cline{1-3}
\multirow{ 2}{*} {$3\%$ (missing value = $0$)} & New approx. & \textbf{95.65\%} \\ \cline{2-3}
 & Structural\ EM  & 56.52\% \\
\hline
\multirow{ 2}{*} {$5\%$ (missing value = $0$)} &New approx. & \textbf{80.43\%} \\ \cline{2-3}
& Structural\ EM  & 39.13\% \\
\hline
\multirow{ 2}{*} {$10\%$ (missing value = $0$)} & New approx. & \textbf{92.75\%} \\ \cline{2-3}
& Structural\ EM  & 60.87\% \\
\hline
\cline{1-3}
\multirow{ 2}{*} {$3\%$ (missing value = $1$)} & New approx. & \textbf{67.83\%} \\ \cline{2-3}
 & Structural\ EM  & 63.48\% \\
\hline
\multirow{ 2}{*} {$5\%$ (missing value = $1$)} & New approx. & \textbf{70\%} \\ \cline{2-3}
& Structural\ EM  & 59.4\% \\
\hline
\multirow{ 2}{*} {$10\%$ (missing value = $1$)} & New approx. & \textbf{71.13\%} \\ \cline{2-3}
& Structural\ EM  & 63.2\% \\
\hline
  \end{tabular}
\end{table}

\subsection{Smoking Cessation Study Dataset}
The dataset used in this experiment is taken from a smoking cessation study as described in \cite{Gruder1993}. It has been further utilized in other works, most notably \cite{Hedeker2007}. The smoking cessation dataset is a binary dataset consisting of $489$ patient records (instances) with the missing data being inherently therein, i.e. there is no need to simulate missing data. The dataset contains $4$ variables including the class variable, which refers to \texttt{smoking} or \texttt{non-smoking}. All the missing values are located in the class variable. There is a total of $372$ patient records with observed classes, consisting of $294$ smoking and $78$ non-smoking records, as well as $117$ records with missing class labels. 

The experiment we perform here is a semi-supervised learning (SSL) experiment where we evaluate the performance of the algorithms as follows: (i) We hide the class labels of a portion of the observed labels; (ii) We apply the approximate learning on the data consisting of the originally missing and artificially hidden labels as missing values, and the rest of the data as observed values. Clearly this is a SSL experiment where the training data consists of the records with observed labels as labeled instances, records with originally missing labels as unlabeled instances, and the test instances are the records with artificially hidden labels.

The evaluation metric is the accuracy of the test instances using a cross-validation approach, as usually done in classification experiments. We compare the performance of the approximate algorithm against an equivalent procedure using structural EM (labels are then chosen based on the posterior distribution), and also against a semi-supervised learner in the form of a Laplacian SVM \citep{Melacci2011} whose code is available online. Accuracies of the approximate algorithm, structural EM, and the semi-supervised Laplacian SVM are displayed in Table~\ref{bn-res04}. Results suggest that the new algorithm is a very promising approach for SSL.

\begin{table}[H]
  \caption{MNAR Semi-supervised learning (SSL) results of the Smoking Cessation study data. All test records are Smoking records. The first column refers to the number of missing values in the test set. Accuracy expresses cross-validated accuracy of the test set.}
  \label{bn-res04}
  \centering
  \begin{tabular}{|c|c|c|c|}
\hline
  \# missing values \ & Algorithm  & Avg.\ Accuracy \\
\hline
\multirow{ 2}{*} {25} & Approx. Learning &  \textbf{90\%} \\ \cline {2-3}
& Structural\ EM  & 15\% \\ \cline{2-3}
& Laplacian SVM &  76\% \\
\hline
\multirow{ 2}{*} {50} & Approx. Learning &  \textbf{88\%} \\ \cline {2-3}
& Structural\ EM  & 10\% \\ \cline{2-3}
& Laplacian SVM &  73.5\% \\
\hline
\multirow{ 2}{*} {75} & Approx. Learning &  \textbf{88\%} \\ \cline {2-3}
& Structural\ EM  & 8\% \\ \cline{2-3}
& Laplacian SVM &  76\% \\
\hline
  \end{tabular}
\end{table}

\subsection{Car Evaluation Dataset}
The Car Evaluation dataset \citep{Blake1998,Lichman2013} contains $1728$ instances and $7$ variables consisting of $6$ attributes and a class. The $6$ attributes refer to the following: buying, maintenance, doors, persons, luggage boots and safety. The class variable refers to the car acceptability and can have exactly one of the following values: \texttt{unacceptable, acceptable, good, very good}. All variables are categorical with 3 or 4 states. The data were derived from a hierarchical decision model originally developed by \cite{Bohanec1988}. Similar to the LUCAS experiment, a MNAR classification task is performed by involving missing values belonging all to one category of the class variable at a time (this is repeated for each label). Due to the class label unbalance (\texttt{unacceptable}: 1210 instances, \texttt{acceptable}: 384, \texttt{good}: 69, \texttt{v-good}: 65), we performed $10$ experiments testing only the {\texttt{unacceptable}} and \texttt{acceptable} labels in five each, where there are $100$ randomly chosen instances with a missing label (test set) in each experiment. The proposed algorithm is compared to structural EM and to an SVM classifier. Classification results are displayed in Table~\ref{bn-res05}. Again, the new algorithm is significantly better than the others.
\begin{table}[H]
  \caption{Accuracy of classification for experiments performed on the Car Evaluation dataset. SVM with an rbf kernel is reported since it leads to best accuracy compared to other $5$ experimented kernels.}
  \label{bn-res05}
  \centering
  \begin{tabular}{|c|c|}
\hline
  Algorithm  & Avg.\ Accuracy \\
\hline
Approximate Learning & \textbf{87.5\%} \\
\hline
Structural\ EM  & 69.38\% \\
\hline
SVM (rbf)  & 85.96\% \\
\hline
  \end{tabular}
\end{table}

\section{Conclusions}
In this paper we discuss the Bayesian network structure learning problem with missing data. We present an approach which performs well even when data are not missing at random. We define an optimization task to tackle the problem and propose a new
exact algorithm for it which translates the task into a structure learning problem without missing data. Inspired by the exact procedure, we develop an approximate algorithm which employs structure optimization as a subcall. Experiments show the advantages of such approach. The proposed approximate method can scale to domains with hundreds or even thousands of variables. We intend to investigate such avenue in future work.

\FloatBarrier

\setlength{\bibsep}{2pt plus 0.3ex}

\bibliographystyle{elsarticle-harv}


\bibliography{bn_refs}

\end{document}